\documentclass[11pt,letterpaper]{article}

\usepackage{arxiv}

\usepackage[T1]{fontenc}
\usepackage[utf8]{inputenc}

\usepackage{amsmath,amssymb,amsthm,mathtools}
\usepackage{bm}
\usepackage{graphicx}
\usepackage{booktabs}
\usepackage{array}
\usepackage{natbib}
\usepackage{microtype}
\usepackage[hidelinks]{hyperref}
\usepackage{url}

\renewcommand{\headeright}{A Preprint}
\renewcommand{\undertitle}{A Preprint}
\renewcommand{\shorttitle}{Transient Amplification in Cascaded Gradient Methods}

\theoremstyle{plain}
\newtheorem{theorem}{Theorem}
\newtheorem{lemma}[theorem]{Lemma}
\newtheorem{corollary}[theorem]{Corollary}
\newtheorem{proposition}[theorem]{Proposition}
\theoremstyle{remark}
\newtheorem{remark}[theorem]{Remark}
\theoremstyle{definition}
\newtheorem{definition}[theorem]{Definition}
\newtheorem{assumption}{Assumption}

\graphicspath{{figures/}}

\newcommand{\Amp}{\mathrm{Amp}}
\newcommand{\Ampstar}{\Amp^{\star}}
\newcommand{\spec}{\mathrm{spec}}
\newcommand{\Jcal}{\mathcal{J}}

\newcommand{\norm}[1]{\left\lVert #1\right\rVert}
\newcommand{\abs}[1]{\left\lvert #1\right\rvert}

\begin{document}

\title{Transient Amplification in Cascaded Gradient Methods: \\ Exact Minimax Laws and Tracking without Resonance}

\author{Ahanaf Hasan Ariq \\
Daffodil International School
\And
Davide Wiest \\
Technische Universität Darmstadt}

\date{}

\maketitle

\begin{abstract}
Leader--follower gradient iterations, as in bilevel optimization and two-time-scale learning, can have large finite-time excursions even when every eigenvalue lies strictly inside the unit disk. We give an exact, dimension-free minimax characterization of this effect for block-triangular (cascaded) iterations
$J=\begin{bmatrix} A & 0\\ C & D\end{bmatrix}$
with symmetric diagonal blocks satisfying $\norm{A}_2,\norm{D}_2\le\gamma<1$ and $\norm{C}_2\le c$. At every horizon $t$, the worst-case amplification over all dimensions equals the spectral norm of the scalar matrix
$\begin{bmatrix} \gamma^t & 0\\ ct\gamma^{t-1} & \gamma^t\end{bmatrix}$,
and an aligned two-dimensional quadratic attains it. The worst transient is therefore $c\max_{t\ge1}t\gamma^{t-1}=\Theta(c/(1-\gamma))$ up to an additive one. We prove an impossibility result for predictions from eigenvalues alone, and martingale noise entering the leader of the aligned instance produces a stationary follower variance of $\sigma^2c^2(1+\gamma^2)/(1-\gamma^2)^3$, cubic in the inverse gap. For quadratic follower problems, gradient structure gives $C=(I-D)R$, where $R$ is the sensitivity of the follower's solution map. We prove that the amplification is then at most $1+\norm{R}_2$ for an isotropic leader and a positive semidefinite follower block, independently of the gaps. A follower that takes one Newton step at the updated leader iterate (look-ahead Newton) removes the follower memory that causes resonance and minimizes the worst-case transient among linear updates that track the leader exactly. Numerically, the constants hold to round-off and each assumption of the sensitivity bound proves necessary. Unexpectedly, accelerated solvers such as Chebyshev iteration exceed the bound by up to 36\%, because their follower map has negative eigenvalues. On a real-data bilevel task, standard solvers lag the follower's optimum in proportion to $\norm{R}_2$ while look-ahead Newton tracks exactly. Code and experiment artifacts are available at \url{https://github.com/Axym-Labs/transient-amplification}.
\end{abstract}

\section{Introduction}

Many learning algorithms update two variables on coupled time scales. Examples include inner--outer methods for bilevel optimization, actor--critic recursions, auxiliary-variable methods, and leader--follower learning. Near a solution, and exactly for quadratic objectives, a one-way coupled method has the form
\begin{equation}
e_{t+1}=Je_t,\qquad J=\begin{bmatrix} A & 0\\ C & D\end{bmatrix},\qquad e_t=(x_t,y_t).
\label{eq:cascade}
\end{equation}
The eigenvalues of $J$ are those of $A$ and $D$, so $\rho(J)<1$ proves asymptotic convergence but says nothing quantitative about the largest value of $\norm{J^t}_2$ before decay. This gap matters when only finitely many iterations are available and when stochastic perturbations are repeatedly amplified.

Transient growth is classical in non-normal linear systems \citep{kreiss1962,trefethen2005}. It has also been identified in accelerated first-order optimization, where momentum creates non-normal state-space dynamics \citep{mohammadi2023}. Our question is different: what is the exact worst transient caused by a one-way coupling, and which part of it is removable by an algorithm that respects the underlying quadratic follower problem?

The distinction matters because a generic matrix class treats $C$ as independent of $D$ and suggests the scale $\norm{C}_2/(1-\gamma)$, where $\gamma<1$ bounds $\norm{A}_2$ and $\norm{D}_2$. In an exact quadratic follower update, however, $C$ and $D$ satisfy the identity $C=(I-D)R$, where $R$ is the derivative of the follower solution map. A useful theory should therefore exploit this to characterize the generic worst case sharply and then expose what the optimization structure changes.

\paragraph{Contributions.}
\begin{enumerate}
\item We prove an exact finite-horizon minimax theorem over all dimensions (Theorem~\ref{thm:minimax}). For symmetric contractions $A$, $D$ and bounded coupling $C$, the worst value of $\norm{J^t}_2$ is attained by a scalar aligned instance. This replaces dimension-dependent Kreiss upper bounds by an equality.
\item We derive a sharp scaling law and a valid spectral indistinguishability result (Corollary~\ref{cor:scaling}, Theorem~\ref{thm:spectral}). Two matrices with identical eigenvalues can have transient amplifications separated by $\Omega(c/(1-\gamma))$.
\item We give an exact stochastic consequence (Propositions~\ref{prop:secondmoment} and~\ref{prop:noisefloor}). In the aligned model, the follower variance under persistent leader noise scales as $c^2/(1-\gamma)^3$; when $c=(1-\gamma)s$, this becomes the sensitivity scale $s^2/(1-\gamma)$.
\item We use the structural identity $C=(I-D)R$ to distinguish static sensitivity from dynamical resonance (Theorem~\ref{thm:structure}). Look-ahead Newton has no resonance and minimizes the worst-case transient among linear updates that respond exactly to the leader (Theorem~\ref{thm:lnf}).
\item We audit every result numerically (Section~\ref{sec:experiments}). The closed-form values match to round-off, and practical follower solvers can exceed the sensitivity bound when their follower map has negative eigenvalues, which Assumption~\ref{ass:psd} excludes.
\end{enumerate}

\paragraph{Scope.}
We study one-way linear cascades, which describe quadratic leader--follower problems exactly and smooth ones near a solution; Section~\ref{sec:discussion} discusses nonlinear dynamics and two-way feedback as open problems.

\subsection{Related work}

\paragraph{Non-normal transients.}
The gap between asymptotic stability and transient size is classical for non-normal matrices. The Kreiss matrix theorem bounds $\sup_t\norm{J^t}_2$ above and below by the Kreiss constant, with a dimension factor $en$ in the upper bound \citep{kreiss1962,spijker1991}, pseudospectra describe the same effect geometrically \citep{trefethen2005}, and computing or reducing the Kreiss constant has been studied as a control-design problem \citep{apkarian2020}. Cascades are the simplest non-normal structure: in nonlinear control, a stable subsystem driving another can produce large transient peaks \citep{sussmann1991}, and in neural circuit models a hidden feedforward coupling amplifies activity in proportion to its strength without changing any eigenvalue \citep{murphy2009}. For the cascaded class in \eqref{eq:cascade}, we replace these bounds by an exact, dimension-free minimax value at every horizon, attained by a two-dimensional instance.

\paragraph{Transient growth of optimization algorithms.}
\citet{mohammadi2023} quantify the transient growth of accelerated first-order methods on quadratics, where momentum makes the state-space matrix non-normal and repeated eigenvalues produce a Jordan-type peak. Our non-normality comes instead from the coupling between two variables, and the block structure lets us identify the worst case exactly and split it into sensitivity and resonance.

\paragraph{Two-time-scale stochastic approximation.}
Finite-time analyses of linear and nonlinear two-time-scale recursions bound the error through step-size ratios and Lyapunov arguments \citep{konda2004,kaledin2020,doan2023,kong2025}, and the same framework underlies two-time-scale methods for bilevel optimization and actor--critic learning \citep{hong2023}. Their rates hold up to constants that depend on the coupled linearization; the transient constant studied here is one such constant, which eigenvalue-based analyses do not capture, and Proposition~\ref{prop:noisefloor} shows that it can dominate the stationary follower error.

\paragraph{Bilevel optimization.}
Complexity bounds and lower bounds for bilevel optimization concern oracle complexity and conditioning \citep{ghadimi2018,ji2023}. Hypergradient methods differ in how they approximate the inner solution: approximate implicit differentiation solves the inner linear system inexactly \citep{pedregosa2016}, iterative differentiation unrolls inner steps, and the errors of both are analysed through the contraction of the inner map \citep{grazzi2020,ji2021}, while warm-started and amortized variants treat the outer--inner iteration as a singularly perturbed system \citep{arbel2022,dagreou2022}. In this language our follower is the inner solver and $R=(I-D)^{-1}C$ is the implicit derivative, so our results concern the transient of the coupled iteration rather than the accuracy of a single hypergradient.

\paragraph{Leader--follower updates.}
Stackelberg learning dynamics couple leader and follower through the implicit function theorem \citep{fiez2020}, follow-the-ridge corrects the follower step by the sensitivity times the leader displacement \citep{wang2020}, and prediction--correction methods for time-varying optimization track a moving optimum with the same sensitivity \citep{simonetto2016}. The ordering of updates also matters for stability, since simultaneous and alternating gradient steps behave differently on bilinear games \citep{zhang2020}. Look-ahead Newton is the exact-solve instance of such corrected followers; our contribution is its transient analysis and optimality, and the finding that practical approximations can leave the sensitivity bound when their follower map has negative eigenvalues.

\paragraph{Organization.}
Section~\ref{sec:cascades} introduces the cascaded class and our assumptions. Section~\ref{sec:minimax} proves the exact minimax law and shows what eigenvalues cannot see, Section~\ref{sec:noise} adds persistent noise, and Section~\ref{sec:structure} uses quadratic structure to derive the sensitivity bound and the look-ahead Newton follower. Section~\ref{sec:experiments} tests all results numerically, and Appendices~\ref{app:proofs}, \ref{app:extensions} and \ref{app:details} contain the proofs, extensions and experimental details.

\section{Exact Quadratic Cascades}
\label{sec:cascades}

We first work with a matrix class and then connect it to optimization.

\paragraph{Notation.}
Throughout, $\norm{\cdot}_2$ is the Euclidean norm of vectors and the induced operator norm of matrices, $\rho(\cdot)$ is the spectral radius, $\spec(\cdot)$ is the eigenvalue multiset, and $\preceq$ is the Loewner order on symmetric matrices. The amplification $\Amp(J)=\sup_{t\ge0}\norm{J^t}_2$ (Definition~\ref{def:class}), the function $\Phi$ in \eqref{eq:Phi}, and the aligned scalar cascade $J_{\gamma,c}=\begin{bmatrix}\gamma&0\\c&\gamma\end{bmatrix}$ (Theorem~\ref{thm:minimax}) are used throughout.

\begin{assumption}[Cascade]
\label{ass:cascade}
The diagonal blocks $A$ and $D$ are symmetric with $\norm{A}_2,\norm{D}_2\le\gamma<1$, and the coupling satisfies $\norm{C}_2\le c$.
\end{assumption}

\begin{definition}[Cascaded contraction class]
\label{def:class}
For $0\le\gamma<1$ and $c\ge0$, let $\Jcal(\gamma,c)$ contain all real block matrices
\begin{equation}
J=\begin{bmatrix} A & 0\\ C & D\end{bmatrix}
\end{equation}
of arbitrary compatible dimensions that satisfy Assumption~\ref{ass:cascade}. Define
\begin{equation}
\Amp(J)=\sup_{t\in\mathbb{N}_0}\norm{J^t}_2,
\qquad
\Ampstar(\gamma,c)=\sup_{J\in\Jcal(\gamma,c)}\Amp(J).
\end{equation}
\end{definition}

In Sections~\ref{sec:minimax} and~\ref{sec:noise}, the symmetry in Assumption~\ref{ass:cascade} is used only to convert spectral-radius control into the contraction bounds $\norm{A}_2,\norm{D}_2\le\gamma$, and those results remain true if these norm bounds are assumed directly; Theorem~\ref{thm:structure} uses the symmetry of $D$ essentially.

The class arises exactly from the quadratic objectives
\begin{equation}
f(x)=\tfrac12 x^\top H_x x,\qquad g(x,y)=\tfrac12 y^\top Hy+y^\top Gx,
\label{eq:quadratic}
\end{equation}
with $H_x,H$ symmetric positive definite. Simultaneous gradient steps give
\begin{equation}
x_{t+1}=(I-\alpha H_x)x_t,\qquad y_{t+1}=(I-\beta H)y_t-\beta Gx_t,
\label{eq:gradsteps}
\end{equation}
so $A=I-\alpha H_x$, $D=I-\beta H$, and $C=-\beta G$.

For $a,b\ge0$, define
\begin{equation}
\Phi(a,b)=\norm{\begin{bmatrix} a & 0\\ b & a\end{bmatrix}}_2=\frac{\sqrt{b^2+4a^2}+b}{2}.
\label{eq:Phi}
\end{equation}

\section{An Exact Minimax Law for Transients}
\label{sec:minimax}

We first characterize the worst transient over the whole class, before using any optimization structure.

\begin{theorem}[Exact minimax law at every horizon]
\label{thm:minimax}
For every $t\ge1$,
\begin{equation}
\sup_{J\in\Jcal(\gamma,c)}\norm{J^t}_2=\Phi\!\left(\gamma^t,\,ct\gamma^{t-1}\right).
\label{eq:minimax}
\end{equation}
The supremum is attained in dimension two by $J_{\gamma,c}=\begin{bmatrix}\gamma&0\\c&\gamma\end{bmatrix}$.
\end{theorem}

\begin{proof}[Proof sketch]
The key observation is that block-norm domination reduces every member of the class to a nonnegative $2\times2$ majorant whose largest singular value is nondecreasing in each entry, and that the aligned instance $J_{\gamma,c}$ attains this majorant (proof in Appendix~\ref{app:proofs}).
\end{proof}

\begin{corollary}[Sharp scaling in coupling and gap]
\label{cor:scaling}
Let
\begin{equation}
m(\gamma)=\max_{t\ge1}t\gamma^{t-1}.
\end{equation}
Then
\begin{equation}
\max\{1,cm(\gamma)\}\le\Ampstar(\gamma,c)\le1+cm(\gamma).
\label{eq:scaling}
\end{equation}
Moreover, for $\gamma\in[1/2,1)$,
\begin{equation}
\frac{1}{4(1-\gamma)}\le m(\gamma)\le\frac{1}{1-\gamma}.
\label{eq:m}
\end{equation}
Thus $\Ampstar(\gamma,c)=\Theta(1+c/(1-\gamma))$ with universal constants.
\end{corollary}

The two sides of \eqref{eq:scaling} differ by at most one, and \eqref{eq:m} is tight up to a factor of four. Coupling and spectral gap thus enter only through their ratio, so a weakly coupled but slow follower can amplify as much as a strongly coupled fast one.

\begin{corollary}[Deterministic iteration complexity]
\label{cor:complexity}
Let $J\in\Jcal(\gamma,c)$ with $\gamma\in[1/2,1)$ and let $e_{t+1}=Je_t$. It is sufficient to take
\begin{equation}
t\ge\frac{2}{-\log\gamma}\log\left(\frac{\left(1+4c/(1-\gamma)\right)\norm{e_0}_2}{\varepsilon}\right)
\label{eq:complexity}
\end{equation}
to ensure $\norm{e_t}_2\le\varepsilon$. Thus the iteration complexity is $O\big((1-\gamma)^{-1}\log((1+c/(1-\gamma))\norm{e_0}_2/\varepsilon)\big)$.
\end{corollary}

The peak occurs near $t=-1/\log\gamma\approx1/(1-\gamma)$ when $c$ is large and earlier when $c$ is small. In a deterministic iteration the transient therefore costs only a logarithmic delay, of order $\log(1+c/(1-\gamma))/(1-\gamma)$ steps by Corollary~\ref{cor:complexity}, whereas Section~\ref{sec:noise} shows that under persistent noise the same factor enters polynomially.

\begin{figure}[t]
\centering
\includegraphics[width=0.78\textwidth]{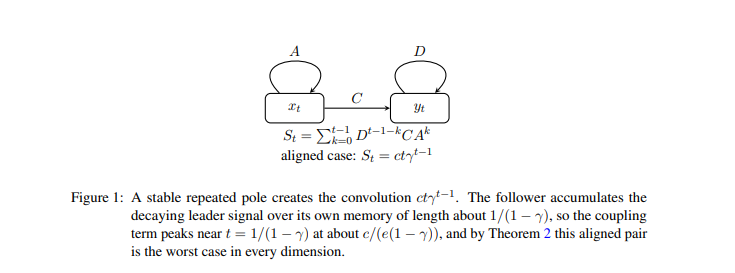}
\caption{A stable repeated pole creates the convolution $ct\gamma^{t-1}$. The follower accumulates the decaying leader signal over its own memory of length about $1/(1-\gamma)$, so the coupling term peaks near $t=1/(1-\gamma)$ at about $c/(e(1-\gamma))$, and by Theorem~\ref{thm:minimax} this aligned pair is the worst case in every dimension.}
\label{fig:cascade}
\end{figure}

\subsection{Eigenvalues are insufficient}

\begin{theorem}[Spectral indistinguishability]
\label{thm:spectral}
Fix $\gamma\in[1/2,1)$ and $c>0$. Let an estimator $q$ depend only on the multiset of eigenvalues. There exist two real $2\times2$ matrices with the same eigenvalue multiset $\{\gamma,\gamma\}$ for which
\begin{equation}
\max_{i\in\{0,1\}}\abs{q(\spec(J_i))-\Amp(J_i)}\ge\frac{\Ampstar(\gamma,c)-1}{2},
\label{eq:spectral}
\end{equation}
where $J_0=\gamma I_2$ and $J_1=J_{\gamma,c}$. In particular, if $cm(\gamma)\ge2$, the right-hand side is at least $cm(\gamma)/4=\Omega(c/(1-\gamma))$.
\end{theorem}

Theorem~\ref{thm:spectral} needs no statistical observation model and means that any guarantee phrased through eigenvalues alone is off by order $c/(1-\gamma)$ on some instance, so a useful certificate has to see how the coupling is oriented relative to the slow modes. The resolvent sees part of this orientation (Section~\ref{sec:kreiss}), and in Section~\ref{sec:structure} the optimization structure supplies the rest by tying the coupling to the follower's sensitivity.

\subsection{Relation to the Kreiss constant}
\label{sec:kreiss}

For $\rho(J)<1$, define $K(J)=\sup_{\abs{z}>1}(\abs{z}-1)\norm{(zI-J)^{-1}}_2$. The Kreiss matrix theorem gives
\begin{equation}
K(J)\le\Amp(J)\le enK(J)
\end{equation}
in dimension $n$ \citep{spijker1991}. In our class, Theorem~\ref{thm:minimax} removes the dimension factor and supplies the exact finite-horizon answer. One might hope to recover Theorem~\ref{thm:minimax} from the resolvent, but even for $J_{\gamma,c}$ the Kreiss constant settles at only about two thirds of the peak, its ratio tending to $e/4$, while the upper bound loses a factor that grows with dimension. This is why we work with the powers directly. The resolvent remains useful computationally when the symmetry or triangular assumptions fail.

\section{Persistent Noise: An Exact Lower Bound}
\label{sec:noise}

We next ask how the transient interacts with noise that enters the leader at every step. Consider the recursion
\begin{equation}
e_{t+1}=Je_t+w_t
\label{eq:noisy}
\end{equation}
under the following noise model.

\begin{assumption}[Martingale noise]
\label{ass:noise}
With $\mathcal{F}_t$ generated by $w_0,\dots,w_{t-1}$, the noise satisfies $\mathbb{E}[w_t\mid\mathcal{F}_t]=0$ and $\mathbb{E}[w_tw_t^\top\mid\mathcal{F}_t]=Q$.
\end{assumption}

\begin{proposition}[Second moment at finite horizon]
\label{prop:secondmoment}
Under Assumption~\ref{ass:noise}, for deterministic $e_0$ and recursion~\eqref{eq:noisy},
\begin{equation}
\mathbb{E}\norm{e_t}_2^2=\norm{J^te_0}_2^2+\sum_{k=0}^{t-1}\operatorname{tr}\!\left(J^kQ(J^k)^\top\right).
\label{eq:secondmoment}
\end{equation}
Consequently, if $Q\preceq\sigma^2I_n$, where $n$ is the dimension of $e_t$, and $J\in\Jcal(\gamma,c)$,
\begin{equation}
\mathbb{E}\norm{e_t}_2^2\le\norm{J^te_0}_2^2+\sigma^2n\sum_{k=0}^{t-1}\Phi\!\left(\gamma^k,ck\gamma^{k-1}\right)^2,
\label{eq:secondmoment-bound}
\end{equation}
with the convention that the $k=0$ term equals one. In particular, letting $t\to\infty$ in the noise term gives the explicit bound
\begin{equation}
\sigma^2n\left(\frac{1}{1-\gamma^2}+\frac{2c^2(1+\gamma^2)}{(1-\gamma^2)^3}\right).
\label{eq:noise-bound}
\end{equation}
\end{proposition}

By \eqref{eq:secondmoment}, every noise injection excites its own transient, so the noise floor scales like the squared peak $(c/(1-\gamma))^2$ times the memory length $1/(1-\gamma)$, and the next exact calculation shows that this cubic law is attained.

\begin{proposition}[Exact noise floor of the aligned cascade]
\label{prop:noisefloor}
Let $J=J_{\gamma,c}$ and let $w_t=(\xi_t,0)$ satisfy Assumption~\ref{ass:noise} with $Q=\sigma^2\operatorname{diag}(1,0)$. The limiting covariance satisfies
\begin{equation}
\operatorname{Var}(x_\infty)=\frac{\sigma^2}{1-\gamma^2},
\label{eq:varx}
\end{equation}
\begin{equation}
\operatorname{Var}(y_\infty)=\frac{\sigma^2c^2(1+\gamma^2)}{(1-\gamma^2)^3}.
\label{eq:vary}
\end{equation}
Hence fixed $c$ gives follower variance $\Theta(\sigma^2c^2/(1-\gamma)^3)$. If $c=(1-\gamma)s$, where $s$ is a fixed sensitivity, the variance is $\Theta(\sigma^2s^2/(1-\gamma))$.
\end{proposition}

The last reparameterization matters for interpretation, because the cubic inverse-gap law holds for an unstructured class with fixed $c$, whereas in a quadratic follower method $c$ generally shrinks with the follower step size. The follower variance is then of the order of $s^2\operatorname{Var}(x_\infty)$, so the follower simply inherits the leader's own fluctuations through its sensitivity.

\section{Quadratic Structure and Exact Tracking}
\label{sec:structure}

We now use the fact that the follower minimizes a quadratic, which ties the coupling to the follower's sensitivity. For \eqref{eq:quadratic}, the exact follower is
\begin{equation}
y^\star(x)=Rx,\qquad R=-H^{-1}G.
\label{eq:follower}
\end{equation}
Every linear preconditioned follower step
\begin{equation}
y_{t+1}=y_t-B(Hy_t+Gx_t)
\label{eq:precond}
\end{equation}
has $D=I-BH$, $C=-BG$, and therefore the structural identity
\begin{equation}
C=(I-D)R.
\label{eq:identity}
\end{equation}
This is the implicit-function derivative of the follower's fixed point \citep{grazzi2020}, so the generic quantities $c$ and $1-\gamma$ are not independent, and a smaller follower step shrinks both together. This suggests that the transient is governed by $R$ rather than by their ratio, which Theorem~\ref{thm:structure} confirms under the following assumptions.

\begin{assumption}[Structured follower]
\label{ass:structured}
The follower update is consistent with the follower's solution map, $C=(I-D)R$ with $R=-H^{-1}G$.
\end{assumption}

\begin{assumption}[Isotropic leader and positive semidefinite follower]
\label{ass:psd}
The leader is isotropic, $A=aI$ with $a\in[0,1)$, and the follower block is symmetric with $0\preceq D\prec I$.
\end{assumption}

\begin{theorem}[Structure-aware transient bound]
\label{thm:structure}
Under Assumptions~\ref{ass:structured} and~\ref{ass:psd}, for every $t\ge1$ the lower-left block of $J^t$ equals $P_t(D)R$, where
\begin{equation}
P_t(D)=(I-D)\sum_{k=0}^{t-1}D^{t-1-k}a^k,
\label{eq:Pt}
\end{equation}
and $\norm{P_t(D)}_2\le1$. Therefore
\begin{equation}
\Amp(J)\le\Phi(1,\norm{R}_2)\le1+\norm{R}_2,
\label{eq:sensitivity}
\end{equation}
independently of the spectral gaps of $A$ and $D$.
\end{theorem}

\begin{proof}[Proof sketch]
For an isotropic leader the off-diagonal block of $J^t$ collapses to $P_t(D)R$, and the key step is the cancellation $0\le p_t(\lambda)\le1-\lambda^t$ on each eigenvalue $\lambda\in[0,1)$ of $D$, so the coupling never contributes more than $\norm{R}_2$ (proof in Appendix~\ref{app:proofs}).
\end{proof}

We call \eqref{eq:sensitivity} the sensitivity bound; it is consistent with Corollary~\ref{cor:scaling}, because in the scalar resonant case $D=aI=\gamma I$, identity \eqref{eq:identity} gives $c=(1-\gamma)\norm{R}_2$, so $c/(1-\gamma)=\norm{R}_2$. In this case the peak tends to $\norm{R}_2/e$ as $\gamma\to1$ and $\norm{R}_2\to\infty$, so \eqref{eq:sensitivity} is tight up to a factor of $e$, and the worst generic transient is the sensitivity permitted by the quadratic model up to that factor.

\begin{remark}[Indefinite followers]
\label{rem:indefinite}
If $D$ has negative eigenvalues, the polynomial $p_t$ can oscillate and the bound $\norm{P_t(D)}_2\le1$ fails, but only by a factor of two, since for $\lambda<0$ and $a\ge0$
\begin{equation}
\abs{p_t(\lambda)}=\frac{(1-\lambda)\abs{a^t-\lambda^t}}{a-\lambda}\le\frac{(1+\abs{\lambda})(a^t+\abs{\lambda}^t)}{a+\abs{\lambda}}\le1+\abs{\lambda}<2,
\end{equation}
so $\Amp(J)\le\Phi(1,2\norm{R}_2)$, still independent of the gaps.
\end{remark}

\begin{remark}[The other assumptions are needed]
\label{rem:needed}
Symmetry of $D$ and isotropy of $A$ cannot be dropped. The non-symmetric follower block $D=\begin{bmatrix}0&0.9\\0&0\end{bmatrix}$, with $\norm{D}_2=0.9$ and spectrum $\{0,0\}$, together with $A=0.5$ and $R=(0,10)^\top$ gives $\Amp(J)=13.48>\Phi(1,10)=10.10$, and the indefinite leader $A=\operatorname{diag}(0.99,-0.99)$ with $D=\operatorname{diag}(0,0.7)$ and $\norm{R}_2=100$ gives $\Amp(J)=103.7>100.0$; for positive semidefinite leaders the constant fails only marginally (Appendix~\ref{app:details}).
\end{remark}

\subsection{A look-ahead Newton follower}
\label{sec:lnf}

The structural identity suggests eliminating follower memory while preserving the same solution map. After computing $x_{t+1}=Ax_t$, take one exact follower Newton step at the fresh leader iterate:
\begin{equation}
y_{t+1}=y_t-H^{-1}(Hy_t+Gx_{t+1})=Rx_{t+1}.
\label{eq:lnf}
\end{equation}
This Gauss--Seidel ordering is an exact Newton step for the quadratic follower and has zero tracking residual after every update; it is the Newton instance of sensitivity-corrected followers such as follow-the-ridge and prediction--correction tracking \citep{wang2020,simonetto2016}. In large problems, it can be approximated with conjugate gradients or a structured preconditioner; the theorem below concerns the exact benchmark.

\begin{theorem}[Exact tracking without resonance]
\label{thm:lnf}
Let $\norm{A}_2\le\gamma<1$ and let
\begin{equation}
J_{\mathrm{LNF}}=\begin{bmatrix} A & 0\\ RA & 0\end{bmatrix}.
\end{equation}
For every $t\ge1$,
\begin{equation}
\norm{J_{\mathrm{LNF}}^t}_2=\norm{\begin{bmatrix} A^t\\ RA^t\end{bmatrix}}_2\le\gamma^t\sqrt{1+\norm{R}_2^2}.
\label{eq:lnfbound}
\end{equation}
Thus $\Amp(J_{\mathrm{LNF}})=\max\{1,\norm{J_{\mathrm{LNF}}}_2\}\le\max\{1,\gamma\sqrt{1+\norm{R}_2^2}\}$, with equality when $A=\gamma I$.

Moreover, \eqref{eq:lnf} is minimax optimal among linear updates that keep the leader step and respond exactly to the leader, $M=\begin{bmatrix} A & 0\\ RA & Z\end{bmatrix}$ with arbitrary follower memory $Z$, since every such $M$ satisfies $\Amp(M)\ge\Amp(J_{\mathrm{LNF}})$. The look-ahead Newton follower ($Z=0$) thus minimizes the worst-case transient over initial states and horizons within this class.
\end{theorem}

\begin{proof}[Proof sketch]
The powers of $J_{\mathrm{LNF}}$ factor as $\begin{bmatrix} A\\ RA\end{bmatrix}A^{t-1}$, so the peak is reached at $t=1$, and every update in the class agrees with $J_{\mathrm{LNF}}$ on states with zero follower, so its norm is at least $\norm{J_{\mathrm{LNF}}}_2$ (proof in Appendix~\ref{app:proofs}). Both parts are tight: the bound \eqref{eq:lnfbound} holds with equality for $A=\gamma I$, and the lower bound over the class is attained by look-ahead Newton itself. The algorithm does not remove the intrinsic sensitivity $\norm{R}_2$; no exact tracker can do so in the joint Euclidean norm. It removes the follower memory that creates the $t\gamma^{t-1}$ resonance and eliminates tracking error after the update; lagging updates outside the class can have smaller joint peaks but pay with tracking error. This distinction also suggests that a practical approximation should keep the follower map free of negative eigenvalues (Remark~\ref{rem:indefinite}) and only then reduce $\norm{I-BH}_2$, while monitoring the sensitivity $\norm{H^{-1}G}_2$.
\end{proof}

\section{Experiments}
\label{sec:experiments}

The theory is exact, so the experiments test its constants and probe its assumptions, with the full protocol in Appendix~\ref{app:details}. All peaks are computed from matrix powers up to a certified horizon, beyond which no later power can exceed them; stochastic quantities use independent seeds with Student-$t$ $95\%$ intervals.

\paragraph{Exact minimax law.}
The directly computed singular values of $J_{\gamma,c}^t$ match $\Phi(\gamma^t,ct\gamma^{t-1})$ to relative error $1.8\times10^{-15}$ at every horizon, and the maximizing horizons agree in all $154$ configurations (Figure~\ref{fig:minimax}, panel A); embedded and random members of $\Jcal(\gamma,c)$ never exceed the bound (Appendix~\ref{app:details}).

\paragraph{Orientation versus spectral certificates.}
Figure~\ref{fig:minimax} (panels B and C) extends Theorem~\ref{thm:spectral} to families with identical diagonal blocks, full spectra and coupling norm. Rotating a rank-one coupling from the slow mode to a fast mode lowers the peak from $73.8$ to $1$ at $\gamma=0.995$, and random orientations concentrate near $1$ as the dimension grows (median $1.04$ at $n=200$).

\paragraph{Persistent noise.}
The estimated stationary follower variance matches Proposition~\ref{prop:noisefloor} and an independent discrete Lyapunov solution, with ratios $1.00\pm0.02$ throughout (Figure~\ref{fig:minimax}, panel D).

\paragraph{Empirical analysis of the sensitivity bound.}
A random and optimized search beyond Theorem~\ref{thm:structure} (Table~\ref{tab:violations}) finds no violation for positive semidefinite followers and leader eigenvalues up to $0.99$, while near the stability boundary the exact constant fails marginally even for commuting positive semidefinite leaders, and indefinite leaders or followers break it (Remark~\ref{rem:needed}, Appendix~\ref{app:details}).

\begin{figure}[t]
\centering
\includegraphics[width=\textwidth]{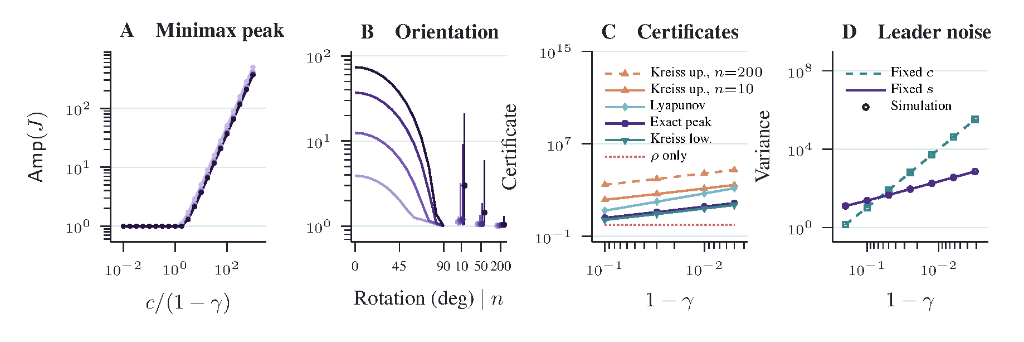}
\caption{(A) Exact peak amplification across coupling-to-gap ratios, from matrix powers up to a certified horizon, for $\gamma\in\{0.5,0.8,0.9,0.95,0.98,0.99,0.995\}$ (darker is larger). (B) Peak as a unit-norm rank-one coupling rotates from the slow ($\gamma$) mode to a fast ($0$) mode, for $\gamma\in\{0.9,0.97,0.99,0.995\}$; bars give the range and dots the median over 30 random orientations in dimension $n$. Spectra and $\norm{C}_2$ are identical across all instances with the same $\gamma$. (C) Lower and upper certificates for slow-mode coupling ($\norm{C}_2=1$). (D) Stationary follower variance under leader noise: closed form (lines; local slopes $-3.00$ and $-1.00$ at the smallest gaps) and simulation (mean and $95\%$ interval over 20 seeds).}
\label{fig:minimax}
\end{figure}

\paragraph{Follower updates at matched computation.}
Figure~\ref{fig:follower} compares follower updates on synthetic quadratic followers with condition numbers $\kappa$ of $H$ from $10$ to $10^4$, at matched numbers of Hessian--vector products per step.
\begin{itemize}
\item The worst-case peak follows the sign of the follower map. Look-ahead Newton attains the value $\gamma\sqrt{1+\norm{R}_2^2}=3.07$ of Theorem~\ref{thm:lnf}. Updates whose effective follower block $D$ has no negative eigenvalues, namely gradient steps of size $1/L$, even numbers of iterations of tuned gradient (the scalar step that minimizes $\norm{I-BH}_2$) or Jacobi, and conjugate gradients (searched worst case), stay below it because a lagging follower keeps $\norm{y_t}_2$ small. Updates with negative eigenvalues in $D$ can exceed both Newton and $\Phi(1,\norm{R}_2)=3.30$, with a single step of tuned gradient or Jacobi reaching $4.2$ and Chebyshev, whose residual polynomial changes sign, reaching $4.5$. This is the negative-eigenvalue boundary of Table~\ref{tab:violations}, observed in practical solvers.
\item Tracking separates the methods. Peak tracking error and stationary tracking MSE improve with budget, fastest for conjugate gradients and Chebyshev, whose advantage shrinks with conditioning: with $32$ iterations per step, the conjugate-gradient tracking error rises from $1.5\times10^{-9}$ at $\kappa=10$ to $0.96$ at $\kappa=10^4$, while look-ahead Newton remains at round-off.
\end{itemize}
Look-ahead Newton thus removes tracking error at the price of exposing the intrinsic sensitivity, as Theorem~\ref{thm:lnf} predicts. Peak joint error alone cannot rank follower updates, because lagging updates have smaller peaks and overshooting updates have larger ones.

\begin{figure}[t]
\centering
\includegraphics[width=\textwidth]{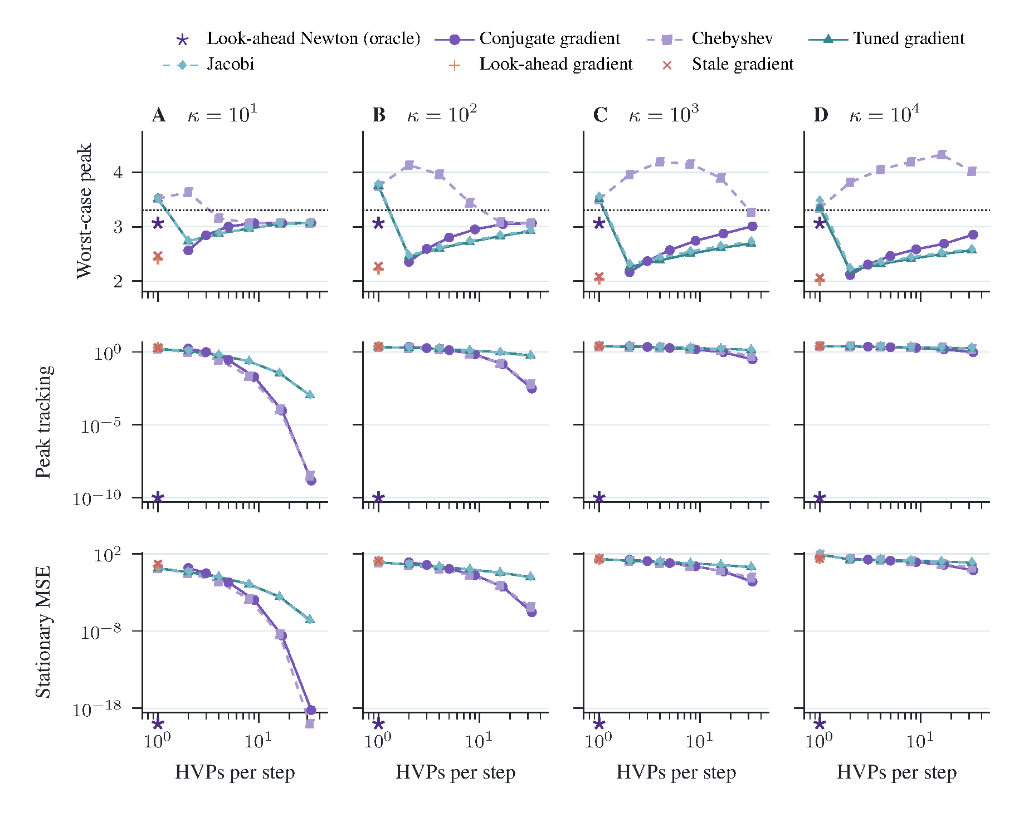}
\caption{Synthetic quadratic followers at matched computation, mean over 5 constructions per condition number $\kappa$ ($\norm{R}_2=3$, leader $A=0.97I$). Rows: largest joint error over initial states (exact for linear updates, searched for conjugate gradients; dotted line $\Phi(1,\norm{R}_2)$), peak tracking error $\norm{y_t-Rx_t}_2$ from the aligned start ($x_0$ the top right singular vector of $R$, $y_0=0$), and stationary tracking MSE. Look-ahead Newton tracks at round-off and is drawn at the axis floor.}
\label{fig:follower}
\end{figure}

\paragraph{Exact bilevel task on real data.}
On a ridge-regression follower whose training labels are shifted by a leader-controlled correction (UCI superconductivity data), the same ranking holds (Figure~\ref{fig:bilevel}). Per unit of leader displacement, tracking error grows linearly with the sensitivity $\norm{R}_2$, while look-ahead Newton tracks exactly (Table~\ref{tab:bilevel}).

\paragraph{Scope of the evidence.}
These experiments match the closed-form values of Theorems~\ref{thm:minimax} and~\ref{thm:lnf} and of Proposition~\ref{prop:noisefloor} to round-off, with the simulated noise floor inside its $95\%$ interval in $27$ of $28$ settings. They also show that the assumptions of Theorems~\ref{thm:structure} and~\ref{thm:lnf} are not merely technical, because indefinite follower updates exceed the sensitivity bound and noisy curvature at high conditioning pushes the look-ahead Newton follower beyond its bound (Appendix~\ref{app:details}).

\begin{figure}[t]
\centering
\includegraphics[width=\textwidth]{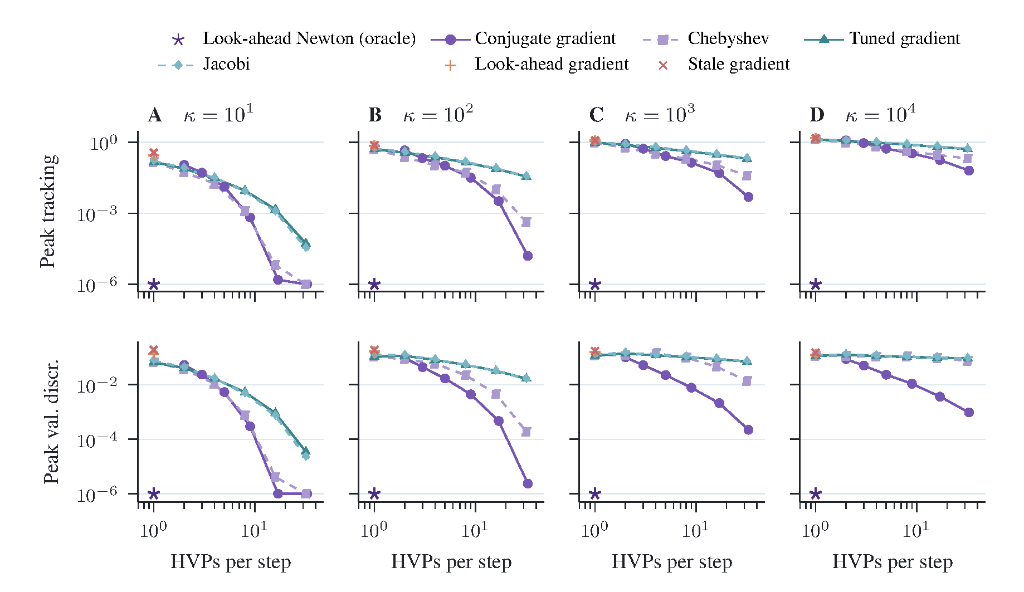}
\caption{Exact bilevel ridge task (UCI superconductivity), follower initialized at $y^\star(x_0)$, mean over 20 seeds, with leader contraction $\norm{A}_2=0.97$ in all settings. Rows: peak tracking error and peak validation discrepancy, per unit leader displacement; the validation discrepancy is the root-mean-square of $X_{\mathrm{val}}(y_t-y^\star(x_t))$.}
\label{fig:bilevel}
\end{figure}

\section{Discussion}
\label{sec:discussion}

Our results separate two sources of transient amplification in leader--follower iterations. The first is resonance, where a follower with slow memory accumulates the leader's signal and amplifies errors by a factor of order $c/(1-\gamma)$. This factor is exact, already attained in two dimensions and invisible to eigenvalues, and under persistent noise it inflates the follower variance to order $c^2/(1-\gamma)^3$. The second is sensitivity, since for an exact quadratic follower the identity $C=(I-D)R$ ties coupling to memory and, for an isotropic leader, caps the amplification at $1+\norm{R}_2$, the size of the follower's response to the leader, which no exact tracker can avoid. Look-ahead Newton removes the resonance with the smallest worst-case transient among exact trackers and leaves only the sensitivity. In our experiments, practical solvers respect the sensitivity bound when their follower map has no negative eigenvalues and can exceed it otherwise, as a single tuned gradient step and Chebyshev iteration do, so a practical follower should keep that map positive semidefinite before it tries to contract faster.

Three questions remain open. First, Theorem~\ref{thm:structure} needs a symmetric follower block and an isotropic leader, and its exact constant fails without isotropy, by a few percent for indefinite leaders and by less than $10^{-4}$ for positive semidefinite ones (Remark~\ref{rem:needed}, Appendix~\ref{app:details}). It is natural to ask whether a constant-factor version survives for positive semidefinite leaders, and the size of these violations suggests that it does. The obstacle is structural, since without isotropy the off-diagonal block becomes the Schur product of $R$ with the kernel $(1-\lambda_i)\sum_k\lambda_i^{t-1-k}\alpha_j^k$ over eigenvalues $\lambda_i$ of $D$ and $\alpha_j$ of $A$, whose entries lie in $[0,1]$ but whose multiplier norm can exceed one, so a uniform bound on that norm seems the most promising route. Second, look-ahead Newton is an oracle benchmark when exact Hessian solves are expensive, and the more practical question is how much of its behaviour survives inexact Krylov solves. Our linear analysis does not apply directly here, because conjugate gradients are not a fixed linear map of the state, so their transient is no longer a matrix power. Interestingly, in our experiments they never exceed the look-ahead Newton value (Section~\ref{sec:experiments}), which hints that a nonlinear transient argument could settle this case. Third, our results concern exact linear and quadratic cascades; local nonlinear versions and two-way feedback, where spectral radius alone does not control the resolvent of a non-normal block, are natural next steps.

\bibliographystyle{plainnat}
\bibliography{references}

\appendix

\section{Proofs}
\label{app:proofs}

The following inequality is used in the proofs of Theorems~\ref{thm:minimax} and~\ref{thm:structure}.

\begin{lemma}[Block domination]
\label{lem:block}
For real matrices $X,Y,Z$ of compatible dimensions,
\begin{equation}
\norm{\begin{bmatrix} X & 0\\ Y & Z\end{bmatrix}}_2\le\norm{\begin{bmatrix} \norm{X}_2 & 0\\ \norm{Y}_2 & \norm{Z}_2\end{bmatrix}}_2.
\end{equation}
\end{lemma}

\begin{proof}
For vectors $u,v$, we have $\norm{Xu}_2\le\norm{X}_2\norm{u}_2$ and
\begin{equation}
\norm{Yu+Zv}_2\le\norm{Y}_2\norm{u}_2+\norm{Z}_2\norm{v}_2,
\end{equation}
hence
\begin{equation}
\norm{\begin{bmatrix} X & 0\\ Y & Z\end{bmatrix}\begin{bmatrix} u\\ v\end{bmatrix}}_2\le\norm{\begin{bmatrix} \norm{X}_2 & 0\\ \norm{Y}_2 & \norm{Z}_2\end{bmatrix}\begin{bmatrix} \norm{u}_2\\ \norm{v}_2\end{bmatrix}}_2,
\end{equation}
and taking the supremum over unit $(u,v)$ gives the claim.
\end{proof}

\subsection*{Proof of Theorem~\ref{thm:minimax}}

Theorem~\ref{thm:minimax} states that, for every $t\ge1$, the largest value of $\norm{J^t}_2$ over the class $\Jcal(\gamma,c)$ is $\Phi(\gamma^t,ct\gamma^{t-1})$, attained by $J_{\gamma,c}$. For $t\ge1$, block multiplication gives
\begin{equation}
J^t=\begin{bmatrix} A^t & 0\\ S_t & D^t\end{bmatrix},
\qquad
S_t=\sum_{k=0}^{t-1}D^{t-1-k}CA^k.
\label{eq:Jt}
\end{equation}
Assumption~\ref{ass:cascade} gives $\norm{A^t}_2,\norm{D^t}_2\le\gamma^t$ and, by the triangle inequality,
\begin{equation}
\norm{S_t}_2\le\sum_{k=0}^{t-1}\gamma^{t-1-k}c\gamma^k=ct\gamma^{t-1}.
\end{equation}
By Lemma~\ref{lem:block}, $\norm{J^t}_2$ is at most the largest singular value of the nonnegative matrix $\begin{bmatrix} \norm{A^t}_2 & 0\\ \norm{S_t}_2 & \norm{D^t}_2\end{bmatrix}$. The largest singular value of a nonnegative $2\times2$ matrix $M$ equals $\max y^\top Mx$ over nonnegative unit vectors $x,y$, so it is nondecreasing in each entry. Applied to \eqref{eq:Jt} with the bounds above, this gives $\norm{J^t}_2\le\Phi(\gamma^t,ct\gamma^{t-1})$, with the convention $0^0=1$ when $\gamma=0$. For $A=D=\gamma$ and $C=c$, we have $J_{\gamma,c}^t=\begin{bmatrix}\gamma^t&0\\ct\gamma^{t-1}&\gamma^t\end{bmatrix}$, so the bound is attained. Finally, for $M=\begin{bmatrix}a&0\\b&a\end{bmatrix}$ with $a,b\ge0$, the matrix $M^\top M=\begin{bmatrix}a^2+b^2&ab\\ab&a^2\end{bmatrix}$ has trace $2a^2+b^2$ and determinant $a^4$, so its largest eigenvalue is
\begin{equation}
\frac{2a^2+b^2+b\sqrt{b^2+4a^2}}{2}=\left(\frac{\sqrt{b^2+4a^2}+b}{2}\right)^2,
\end{equation}
and taking the square root gives \eqref{eq:Phi}.

\subsection*{Proof of Corollary~\ref{cor:scaling}}

Corollary~\ref{cor:scaling} states that $\Ampstar(\gamma,c)$ lies between $\max\{1,cm(\gamma)\}$ and $1+cm(\gamma)$, with $m(\gamma)=\Theta(1/(1-\gamma))$. For $a,b\ge0$, $b\le\Phi(a,b)\le a+b$. The lower inequality follows directly from \eqref{eq:Phi}; the upper inequality uses $\sqrt{b^2+4a^2}\le b+2a$. Taking the supremum of \eqref{eq:minimax} over $t\ge1$ and including $\norm{J^0}_2=1$ gives
\begin{equation}
\Ampstar(\gamma,c)=\max\left\{1,\sup_{t\ge1}\Phi(\gamma^t,ct\gamma^{t-1})\right\},
\end{equation}
and the two inequalities for $\Phi$ bound this quantity below by $\max\{1,cm(\gamma)\}$ and above by
\begin{equation}
\max\left\{1,\sup_{t\ge1}\left(\gamma^t+ct\gamma^{t-1}\right)\right\}\le1+cm(\gamma),
\end{equation}
which is \eqref{eq:scaling}.

For the upper bound in \eqref{eq:m}, $t\gamma^{t-1}\le\sum_{k=0}^{t-1}\gamma^k\le(1-\gamma)^{-1}$. For the lower bound, set $t=\lceil(1-\gamma)^{-1}\rceil$. Then $t\ge(1-\gamma)^{-1}$ and $t-1<(1-\gamma)^{-1}$, so
\begin{equation}
t\gamma^{t-1}\ge\frac{\gamma^{1/(1-\gamma)}}{1-\gamma}\ge\frac{1}{4(1-\gamma)}
\end{equation}
for $\gamma\in[1/2,1)$. The last step holds because $f(\gamma)=\gamma^{1/(1-\gamma)}$ equals $1/4$ at $\gamma=1/2$ and is increasing: $\frac{d}{d\gamma}\log f=h(\gamma)/(1-\gamma)^2$ with $h(\gamma)=(1-\gamma)/\gamma+\log\gamma$, and $h$ decreases on $(0,1)$ to $h(1)=0$, so $h>0$.

\subsection*{Proof of Corollary~\ref{cor:complexity}}

Corollary~\ref{cor:complexity} states that the number of iterations in \eqref{eq:complexity} suffices for $\norm{e_t}_2\le\varepsilon$. Theorem~\ref{thm:minimax} and $\Phi(a,b)\le a+b$ imply $\norm{J^t}_2\le\gamma^t+ct\gamma^{t-1}$. Also
\begin{equation}
t\gamma^{t-1}=\gamma^{t/2}\,t\gamma^{t/2-1}\le\frac{\gamma^{t/2}}{\sqrt{\gamma}(1-\gamma)}\le\frac{4\gamma^{t/2}}{1-\gamma},
\end{equation}
where the first inequality uses
\begin{equation}
t\gamma^{t/2-1}=\frac{t\gamma^{t-1}}{\sqrt{\gamma}}\le\frac{t\gamma^{t-1}}{\gamma}
\qquad\text{and}\qquad
t\gamma^{t-1}\le\sum_{k=0}^{t-1}\sqrt{\gamma}^{\,k}\le\frac{1}{1-\sqrt{\gamma}},
\end{equation}
since every term of the sum is at least $\sqrt{\gamma}^{\,t-1}$, and the second uses $\gamma\ge1/2$ and $1-\sqrt{\gamma}=(1-\gamma)/(1+\sqrt{\gamma})\ge(1-\gamma)/2$. Hence $\norm{J^t}_2\le(1+4c/(1-\gamma))\gamma^{t/2}$. Solving $\norm{J^t}_2\norm{e_0}_2\le\varepsilon$ gives \eqref{eq:complexity}.

\subsection*{Proof of Theorem~\ref{thm:spectral}}

Theorem~\ref{thm:spectral} states that no estimator based on eigenvalues alone predicts $\Amp$ to within $(\Ampstar(\gamma,c)-1)/2$ on both $\gamma I_2$ and $J_{\gamma,c}$. The matrices $J_0$ and $J_1$ have the same eigenvalue multiset, so $q$ outputs the same number on both. Since $\Amp(J_0)=1$ and, by Theorem~\ref{thm:minimax}, $\Amp(J_1)=\Ampstar(\gamma,c)$, the triangle inequality implies that at least one absolute error is at least half their separation. If $cm(\gamma)\ge2$, Corollary~\ref{cor:scaling} gives $(\Ampstar-1)/2\ge(cm-1)/2\ge cm/4$.

\subsection*{Proof of Proposition~\ref{prop:secondmoment}}

Proposition~\ref{prop:secondmoment} gives the exact second moment \eqref{eq:secondmoment} under martingale noise and its bounds over $\Jcal(\gamma,c)$. Unrolling \eqref{eq:noisy} gives $e_t=J^te_0+\sum_{j=0}^{t-1}J^{t-1-j}w_j$. In the expansion of $\mathbb{E}\norm{e_t}_2^2$, the cross terms $\mathbb{E}[(J^te_0)^\top J^{t-1-j}w_j]$ and $\mathbb{E}[(J^{t-1-i}w_i)^\top J^{t-1-j}w_j]$ with $i<j$ vanish after conditioning on $\mathcal{F}_j$, because $w_i$ is $\mathcal{F}_j$-measurable and $\mathbb{E}[w_j\mid\mathcal{F}_j]=0$. The remaining terms are $\mathbb{E}\norm{J^{t-1-j}w_j}_2^2=\operatorname{tr}(J^{t-1-j}Q(J^{t-1-j})^\top)$, and substituting $k=t-1-j$ yields \eqref{eq:secondmoment}.

If $Q\preceq\sigma^2I$, then $\operatorname{tr}(J^kQ(J^k)^\top)\le\sigma^2\norm{J^k}_F^2\le\sigma^2n\norm{J^k}_2^2$. Theorem~\ref{thm:minimax} gives \eqref{eq:secondmoment-bound}. Finally, $\Phi(a,b)^2\le(a+b)^2\le2a^2+2b^2$ and the power-series identities $\sum_{k\ge0}\gamma^{2k}=(1-\gamma^2)^{-1}$ and $\sum_{k\ge1}k^2\gamma^{2k-2}=(1+\gamma^2)/(1-\gamma^2)^3$ give \eqref{eq:noise-bound}.

\subsection*{Proof of Proposition~\ref{prop:noisefloor}}

Proposition~\ref{prop:noisefloor} gives the limiting leader and follower variances of the aligned cascade under leader noise. The effect of a leader impulse after $k$ transitions is
\begin{equation}
J_{\gamma,c}^k\begin{bmatrix}1\\0\end{bmatrix}=\begin{bmatrix}\gamma^k\\ ck\gamma^{k-1}\end{bmatrix},\qquad k\ge1.
\end{equation}
The limiting covariance is $\sum_{k\ge0}J^kQ(J^k)^\top$ with $Q=\sigma^2\operatorname{diag}(1,0)$. Therefore
\begin{equation}
\operatorname{Var}(x_\infty)=\sigma^2\sum_{k\ge0}\gamma^{2k}=\frac{\sigma^2}{1-\gamma^2},
\end{equation}
and
\begin{equation}
\operatorname{Var}(y_\infty)=\sigma^2c^2\sum_{k\ge1}k^2\gamma^{2k-2}=\frac{\sigma^2c^2(1+\gamma^2)}{(1-\gamma^2)^3},
\end{equation}
using $\sum_{k\ge1}k^2r^{k-1}=(1+r)/(1-r)^3$ for $\abs{r}<1$. Equivalently, $y_t$ is an AR(2) process with a double root at $\gamma$, whose stationary variance is classical.

\subsection*{Proof of Theorem~\ref{thm:structure}}

Theorem~\ref{thm:structure} states that, under Assumptions~\ref{ass:structured} and~\ref{ass:psd}, the off-diagonal block of $J^t$ is $P_t(D)R$ with $\norm{P_t(D)}_2\le1$, so that $\Amp(J)\le\Phi(1,\norm{R}_2)$. Substitute $A=aI$ and $C=(I-D)R$ into \eqref{eq:Jt}. Since $aI$ commutes with every matrix,
\begin{equation}
S_t=(I-D)\sum_{k=0}^{t-1}D^{t-1-k}a^kR=P_t(D)R.
\end{equation}
Diagonalize the symmetric matrix $D$; then $P_t(D)$ has the same eigenvectors, and its eigenvalue for an eigenvalue $\lambda\in[0,1)$ of $D$ is
\begin{equation}
p_t(\lambda)=(1-\lambda)\sum_{k=0}^{t-1}\lambda^{t-1-k}a^k.
\end{equation}
Every term is nonnegative because $\lambda,a\ge0$, and $a^k\le1$, hence $0\le p_t(\lambda)\le(1-\lambda)\sum_{j=0}^{t-1}\lambda^j=1-\lambda^t\le1$. Since $P_t(D)$ is symmetric, $\norm{P_t(D)}_2=\max_\lambda\abs{p_t(\lambda)}\le1$, so $\norm{S_t}_2\le\norm{R}_2$. Lemma~\ref{lem:block}, together with $\norm{A^t}_2,\norm{D^t}_2\le1$, gives \eqref{eq:sensitivity}.

\subsection*{Proof of Theorem~\ref{thm:lnf}}

Theorem~\ref{thm:lnf} gives the value of $\Amp(J_{\mathrm{LNF}})$ and its bound \eqref{eq:lnfbound}, and states that no linear update that responds exactly to the leader has smaller amplification. For $t\ge1$, direct multiplication gives
\begin{equation}
J_{\mathrm{LNF}}^t=\begin{bmatrix} A^t & 0\\ RA^t & 0\end{bmatrix}.
\end{equation}
Its norm equals the norm of the stacked operator $\begin{bmatrix}A\\RA\end{bmatrix}A^{t-1}$ in \eqref{eq:lnfbound}; since $\norm{A}_2\le1$, the supremum over $t\ge1$ is attained at $t=1$, giving $\Amp(J_{\mathrm{LNF}})=\max\{1,\norm{J_{\mathrm{LNF}}}_2\}$. For any unit $u$,
\begin{equation}
\norm{A^tu}_2^2+\norm{RA^tu}_2^2\le\gamma^{2t}\left(1+\norm{R}_2^2\right),
\end{equation}
which proves the upper bound. If $A=\gamma I$, equality holds by choosing a right singular vector of $R$.

Finally, every $M=\begin{bmatrix}A&0\\RA&Z\end{bmatrix}$ agrees with $J_{\mathrm{LNF}}$ on states $(x_0,0)$, and $J_{\mathrm{LNF}}$ vanishes on states $(0,y_0)$, so $\norm{M}_2\ge\norm{J_{\mathrm{LNF}}}_2$ and $\Amp(M)\ge\max\{1,\norm{M}_2\}\ge\Amp(J_{\mathrm{LNF}})$.

\section{Additional Checks and Extensions}
\label{app:extensions}

\paragraph{Approximate follower solves.}
Suppose $B$ is chosen so that $D=I-BH$ is symmetric and $0\preceq D\prec I$. Identity \eqref{eq:identity} still holds exactly for the simultaneous update, so Theorem~\ref{thm:structure} applies in the isotropic leader model; the look-ahead ordering gives $C=(I-D)RA$, which is covered with $R$ replaced by $aR$. For a general leader, the look-ahead exact Newton choice retains Theorem~\ref{thm:lnf} without any commutation assumption.

\paragraph{Pseudospectral certificate for the aligned instance.}
For $J_{\gamma,c}$,
\begin{equation}
(zI-J_{\gamma,c})^{-1}=\begin{bmatrix}(z-\gamma)^{-1}&0\\ c(z-\gamma)^{-2}&(z-\gamma)^{-1}\end{bmatrix}.
\end{equation}
Thus its resolvent norm is exactly $\Phi(\abs{z-\gamma}^{-1},c\abs{z-\gamma}^{-2})$ because, with $z-\gamma=\abs{z-\gamma}e^{i\theta}$, multiplying by the diagonal unitaries $\operatorname{diag}(e^{i\theta},e^{2i\theta})$ on the left and $\operatorname{diag}(1,e^{-i\theta})$ on the right makes all entries nonnegative. Since $\abs{z-\gamma}\ge\abs{z}-\gamma$ and $\Phi$ is increasing, the supremum defining $K(J_{\gamma,c})$ is attained on the positive real axis. This gives a one-dimensional, numerically stable computation of $K(J_{\gamma,c})$ and a cross-check for the power-based calculation.

\section{Experimental Details}
\label{app:details}

\paragraph{Certification and replication.}
Every peak $\sup_t\norm{J^t}_2$ for a member of $\Jcal(\gamma,c)$ is computed from matrix powers up to a certified horizon, where the search stops once the tail bound $\gamma^t+ct\gamma^{t-1}\ge\norm{J^t}_2$, which decreases beyond its maximizer, falls below the observed peak, so no later horizon can exceed it. For general joint iterations, the search stops at the first $T$ with $\norm{J^T}_2<1$, because $J^{kT+r}=(J^T)^kJ^r$ then bounds every later power by the running maximum. Stochastic quantities use independent seeds as the unit of replication, with Student-$t$ $95\%$ intervals across seeds. Code, seeds and per-configuration CSV files, including the verified norms, spectra and symmetry residuals, are released at \url{https://github.com/Axym-Labs/transient-amplification}.

\paragraph{Protocol.}
\begin{enumerate}
\item \emph{Exact minimax sweep.} We swept $\gamma\in\{0.5,0.8,0.9,0.95,0.98,0.99,0.995\}$ and $c/(1-\gamma)$ over five decades ($21$ log-spaced values in $[10^{-2},10^{3}]$, plus $c=0$). We compared the observed $\norm{J^t}_2$ with \eqref{eq:minimax} at every horizon and report the maximizing horizon and the relative error. We also embedded $J_{\gamma,c}$ in random symmetric blocks of dimension $10$, $50$ and $200$, and drew $630$ random members of $\Jcal(\gamma,c)$ with dense couplings.
\item \emph{Alignment ablation.} In dimensions $10$, $50$ and $200$, with $A=D=\operatorname{diag}(\gamma,0,\dots,0)$, we compared rank-one couplings aligned with the slow modes, aligned with fast modes, rotated continuously between the two, and in $30$ random orientations, all rescaled to the same operator norm. We report the actual $\norm{C}_2$, not a nominal generation scale, and verified that the spectra coincide within each family.
\item \emph{Eigenvalue twins.} Within the same family, fast-mode coupling has the same peak as $\gamma I$, and slow-mode coupling reproduces $J_{\gamma,c}$, so the two have identical eigenvalues but different transients. For every instance we computed the spectral-radius prediction, the Kreiss lower and dimension-dependent upper certificates, a discrete Lyapunov upper certificate and the exact peak.
\item \emph{Stochastic scaling.} We estimated the stationary follower variance from independent chains and compared it with \eqref{eq:vary} and with a discrete Lyapunov solution. Each of $20$ seeds simulated $2000$ chains for a burn-in of $20/(1-\gamma)$ steps, and doubling the burn-in served as a check. Log--log regressions over the gap test the predicted slopes $-3$ for fixed $c$ and $-1$ for fixed sensitivity $s$ with $c=(1-\gamma)s$.
\item \emph{Follower preconditioning.} We used synthetic quadratic followers with $n=100$, leader $A=0.97I$ and $5$ constructions per setting, varying the condition number from $10$ to $10^4$ at $\norm{R}_2=3$ and then $\norm{R}_2\in\{0.3,1,3,10\}$ at condition number $10^3$. On these we compared stale and look-ahead gradient steps; $k\in\{1,2,4,\dots,32\}$ look-ahead iterations of tuned gradient (the scalar step that minimizes $\norm{I-BH}_2$), Jacobi, Chebyshev and conjugate gradients, each with warm and cold starts; and the look-ahead Newton follower. We report the largest joint error over all initial states, peak tracking error $\norm{y_t-Rx_t}_2$, stationary tracking MSE, Hessian--vector products and setup costs. All methods share the leader trajectory, initial states and noise realizations.
\item \emph{Exact bilevel task.} We used ridge regression whose training labels are shifted by a leader-controlled correction, on the UCI superconductivity data ($21{,}263$ materials, $81$ features). The follower $g(x,y)=\tfrac12\norm{X_{\mathrm{tr}}y-(b_{\mathrm{tr}}+Ux)}_2^2+\tfrac{\lambda}{2}\norm{y}_2^2$ is quadratic, $R=H^{-1}X_{\mathrm{tr}}^\top U$ is exact, and the leader descends the exact reduced validation objective. We varied the inner condition number while holding $\norm{R}_2$ fixed, then varied $\norm{R}_2$ while holding the condition number fixed, with $20$ seeds per setting, each fixing the data split and the construction. This separates dynamical conditioning from sensitivity.
\end{enumerate}

\paragraph{Minimax audit.}
The direct matrix-power calculation matches Theorem~\ref{thm:minimax} to relative error $1.8\times10^{-15}$. Embedding $J_{\gamma,c}$ in random symmetric blocks of dimension up to $200$ reproduces the scalar peak to $4\times10^{-14}$, and none of the $630$ random members of $\Jcal(\gamma,c)$ exceeds the horizon-wise bound; the largest observed ratio, $0.99995$, shows that the bound is nearly attained without alignment at small horizons.

\paragraph{Certificates.}
At $(\gamma,c)=(0.97,1)$, the aligned peak is $12.462$ in dimensions $10$, $50$ and $200$, while $30$ random orientations of the same coupling give peaks from $1.005$ to $3.19$, and for the eigenvalue twins at $(\gamma,c)=(0.99,1)$ the peaks are $1$ and $36.977$ despite identical spectra. All $540$ alignment configurations respect the ordering lower certificate $\le$ peak $\le$ Lyapunov $\le$ Kreiss upper. A prediction from the spectral radius alone cannot distinguish any of these matrices, the Kreiss lower bound recovers about two thirds of the aligned peak, the dimension-dependent Kreiss upper bound grows linearly with $n$ although the exact peak does not, and the identity-weighted Lyapunov certificate exceeds the peak by $4$--$19\times$. The Kreiss constant is computed on a log-polar grid with local refinement and, for slow-mode coupling, matches the exact one-dimensional expression from the resolvent formula of Appendix~\ref{app:extensions} to $2\times10^{-14}$.

\paragraph{Persistent noise.}
The local log--log slopes at the smallest gaps are $-2.998$ for fixed $c$ and $-0.998$ for fixed $s$, approaching the asymptotic exponents $-3$ and $-1$. Over $1-\gamma\in[0.003,0.2]$, the seed-level log--log slopes are $-2.978\pm0.003$ for fixed $c=0.2$ and $-0.975\pm0.005$ for fixed $s=3$. The exact finite-horizon covariance at the end of the burn-in equals its stationary value to displayed precision.

\paragraph{Sensitivity bound.}
With $B=(1-\gamma)H^{-1}$, the identities $D=\gamma I$ and $C=(I-D)R$ hold to $10^{-15}$. The peak-to-bound ratio decreases to $0.370\approx1/e$, the value implied by $\max_t(1-\gamma)t\gamma^{t-1}\to1/e$. Table~\ref{tab:violations} reports the search beyond the assumptions of Theorem~\ref{thm:structure}, which uses random instances in dimensions $2$, $5$ and $20$ and Nelder--Mead maximization of the observed-to-certified ratio in dimensions $2$ to $4$; all optimized runs with $0\preceq D\prec I$ converge to the same value, $0.995$. These searches use positive semidefinite leaders; an indefinite leader breaks the bound, e.g.\ $A=\operatorname{diag}(0.99,-0.99)$, $D=\operatorname{diag}(0,0.7)$ and $R=100\begin{bmatrix}0.566&0.824\\0.824&-0.566\end{bmatrix}$ give $\Amp(J)=103.7>\Phi(1,\norm{R}_2)=100.0$ at $t=6$. Positive semidefinite leaders break the exact constant only marginally and only near the stability boundary, e.g.\ the commuting blocks $D=\operatorname{diag}(0,0.0128,0.4728)$ and $A=\operatorname{diag}(0.7169875,0.9354439,0.9999988)$ with $R=10^4X/\norm{X}_2$,
\begin{equation}
X=\begin{bmatrix}-0.027&0.779&0.588\\ -0.321&-0.270&0.009\\ -0.024&0.563&-0.809\end{bmatrix},
\end{equation}
give $\Amp(J)=10000.33>\Phi(1,\norm{R}_2)=10000.0001$ at $t=14$; the searches above use leader eigenvalues up to $0.99$ and $\norm{R}_2\le100$, where no such violation appears. With negative follower eigenvalues, $\norm{I-D}_2$ can approach $1+\gamma$ at $t=1$, which matches Remark~\ref{rem:indefinite}. For the look-ahead Newton follower, truncated conjugate-gradient solves stay below $\gamma\sqrt{1+\norm{R}_2^2}$ because they lag, and with float32 arithmetic the peak exceeds it by at most $0.06\%$, at condition number $10^6$, whereas with $1\%$ relative noise in Hessian--vector products the peak exceeds it by $1.74\times$ at condition number $10^3$.

\begin{table}[t]
\centering
\caption{Search for violations of $\Amp(J)\le\Phi(1,\norm{R}_2)$ with $C=(I-D)R$ and $\norm{A}_2,\norm{D}_2\le0.99$. Ratios are observed peak over certified bound; non-isotropic leaders are positive semidefinite; only the first row is covered by Theorem~\ref{thm:structure}.}
\label{tab:violations}
\small
\begin{tabular}{llcccc}
\toprule
Leader $A$ & Follower $D$ & Instances & Random & Optimized & Violations \\
\midrule
isotropic       & $\succeq0$   & 648 & 0.967 & 0.995 & 0 \\
commuting       & $\succeq0$   & 648 & 0.980 & 0.995 & 0 \\
non-commuting   & $\succeq0$   & 648 & 0.957 & 0.995 & 0 \\
isotropic       & indefinite   & 616 & 1.957 & 1.975 & 319 \\
non-commuting   & indefinite   & 616 & 1.964 & 1.975 & 325 \\
\bottomrule
\end{tabular}
\end{table}

\paragraph{Follower updates.}
For updates that are fixed linear maps, the largest joint error over initial states is exact: $\sup_t\norm{J^t}_2$ of the joint iteration, certified as above. Conjugate gradients are not linear in the state, so for them it is the largest peak over an aligned start, $64$ random starts and the worst-case starts of all linear methods. On linear methods, this search recovers a median $99.9\%$ (minimum $89\%$) of the exact value without access to the method's own worst-case start. Chebyshev iterations use the exact extreme eigenvalues of $H$, and spectral bounds and factorizations are charged as setup cost. Stationary tracking MSE is exact (discrete Lyapunov) for linear updates and simulated over $20$ noise seeds for conjugate gradients, where doubling the burn-in changes it by at most $2\%$. In the worst-case row of Figure~\ref{fig:follower}, the gradient step of size $1/L$ reaches at most $2.5$, and even numbers of iterations of tuned gradient or Jacobi and conjugate gradients reach at most $3.07$. Look-ahead Newton has relative solve residual at most $1.1\times10^{-15}$ and tracks the follower solution to within $1.4\times10^{-12}$ at every tested condition number, on synthetic and real data.

\paragraph{Exact bilevel task.}
An invertible feature transform sets $\kappa(H)$ exactly without changing the unregularized fit, $U$ sets $\norm{R}_2$ exactly, and the leader contraction is fixed at $\norm{A}_2=0.97$ in every setting. Because the leader optimum shrinks as $1/\norm{R}_2$, the follower's total displacement $\norm{R(x_0-x^\star)}_2$ averages between $5.1$ and $5.6$ across the sensitivity sweep with no trend, so we report errors per unit of leader displacement $\norm{x_0-x^\star}_2$. Per unit of leader displacement, the peak tracking error is about $0.4\norm{R}_2$ for the stale gradient step and $0.05\norm{R}_2$ for eight conjugate-gradient iterations (Table~\ref{tab:bilevel}), and with eight iterations per step conjugate gradients track $13.6\times$ better than tuned gradient at $\kappa=10$ and $2.4\times$ better at $\kappa=10^4$. The exact follower's test error is $17.5$--$18.0$~K root-mean-square, and from the converged start every method's final test error agrees with it to within $0.02$~K.

\begin{table}[t]
\centering
\caption{Varying the sensitivity at fixed conditioning $\kappa=10^3$. Synthetic entries are means over 5 constructions; real-data entries are means $\pm$ 95\% intervals over 20 seeds.}
\label{tab:bilevel}
\small
\begin{tabular}{ccccccc}
\toprule
& \multicolumn{3}{c}{Synthetic: worst-case peak / tracking} & \multicolumn{3}{c}{Real data: tracking per $\norm{x_0-x^\star}$} \\
\cmidrule(lr){2-4}\cmidrule(lr){5-7}
$\norm{R}$ & Stale & CG-8 & Newton & Stale & CG-8 & Newton \\
\midrule
0.3 & 1.00 / 0.26 & 1.00 / 0.15 & 1.01 / 0 & $0.12 \pm 0.01$ & $0.01 \pm 0.00$ & 0 \\
1   & 1.12 / 0.85 & 1.27 / 0.50 & 1.37 / 0 & $0.39 \pm 0.03$ & $0.05 \pm 0.01$ & 0 \\
3   & 2.10 / 2.58 & 2.70 / 1.60 & 3.07 / 0 & $1.15 \pm 0.07$ & $0.14 \pm 0.02$ & 0 \\
10  & 6.37 / 8.62 & 8.57 / 5.23 & 9.75 / 0 & $4.17 \pm 0.25$ & $0.51 \pm 0.07$ & 0 \\
\bottomrule
\end{tabular}
\end{table}

\end{document}